\documentclass[11pt,letterpaper]{article}

\usepackage[T1]{fontenc}
\usepackage{lmodern}
\usepackage[margin=1in]{geometry}
\usepackage{amsmath,amssymb,amsthm}
\usepackage{microtype}
\usepackage{booktabs}
\usepackage[round,authoryear]{natbib}
\usepackage{xcolor}
\usepackage{adjustbox}
\usepackage{hyperref}
\hypersetup{
  colorlinks=true,
  linkcolor=blue,
  citecolor=blue,
  urlcolor=blue,
  pdftitle={Polylogarithmic Nash Regret in Matrix Games with Bandit Feedback},
  pdfauthor={Yuheng Zhang}
}
\newtheorem{theorem}{Theorem}
\newtheorem{lemma}[theorem]{Lemma}

\newcounter{algorithm}
\newenvironment{algobox}[1]{%
  \begin{figure}[!t]
  \refstepcounter{algorithm}%
  \begin{adjustbox}{max totalsize={\textwidth}{0.92\textheight},center}
  \begin{minipage}{\textwidth}
  \small
  \setlength{\abovedisplayskip}{3pt}
  \setlength{\belowdisplayskip}{3pt}
  \setlength{\abovedisplayshortskip}{3pt}
  \setlength{\belowdisplayshortskip}{3pt}
  \setlength{\parindent}{0pt}
  \hrule\smallskip
  \textbf{Algorithm \thealgorithm. #1}\par\smallskip
  \hrule\smallskip
}{%
  \par\smallskip\hrule
  \end{minipage}
  \end{adjustbox}
  \end{figure}
}

\title{Polylogarithmic Nash Regret in Matrix Games \\ with Bandit Feedback}

\author{Yuheng Zhang\\
University of Illinois Urbana-Champaign\\
\texttt{yuhengz2@illinois.edu}}
\date{}

\begin{document}

\maketitle

\begin{abstract}
We study Nash regret minimization in unknown finite matrix games with
bandit payoff feedback and observed opponent actions.
We develop Optimistic Payoff Balancing (OPB), which achieves
instance-dependent $\mathcal{O}(\log^2 T)$ Nash regret against arbitrary
adaptive opponents, including games with nonunique equilibria.
This resolves the open problem posed by \citet{maiti2025limitations},
extending their polylogarithmic guarantee under bandit feedback from
$2\times2$ games to arbitrary finite dimensions.
To handle nonunique equilibria, we construct a reference strategy that
leaves room for local adjustments.
We order independent payoff differences by estimation accuracy and scale
these adjustments by uncertainty, allowing the learner to exploit the
opponent's imbalance to offset estimation costs.
Our result thus shows that observing opponent actions suffices for
polylogarithmic Nash regret in general finite matrix games.
\end{abstract}

\section{Introduction}\label{sec:introduction}

We study Nash regret minimization in an unknown two-player zero-sum
matrix game against an adaptive opponent. The learner controls one
player and, after each interaction, observes the opponent's action and
a noisy payoff for the sampled action pair. The value $v(A)$ of the
payoff matrix $A$ is the expected reward that an equilibrium strategy
guarantees against every opponent. When the matrix is unknown, the
learner must balance learning its payoffs with earning reward against
the opponent. Nash regret measures the cumulative reward shortfall
relative to the game value \citep{o2021matrix,maiti2025limitations}:
\[
  R_T=Tv(A)-\mathbb E\!\left[\sum_{t=1}^T r_t\right],
\]
where $r_t$ is the observed reward at round $t$.
A small Nash regret therefore guarantees cumulative reward close to
the game's value throughout learning, even when the opponent follows
an arbitrary strategy.

\begin{table}[!t]
\centering
\small
\setlength{\tabcolsep}{3pt}
\renewcommand{\arraystretch}{1.15}
\caption{Nash regret guarantees against arbitrary opponents in unknown finite
matrix games.
\emph{Actions} means that the opponent's realized actions are observed;
\emph{bandit only} reveals only the sampled payoff.
NE denotes a Nash equilibrium. A pure NE is strict if every unilateral
deviation to another pure action strictly lowers the deviating player's payoff.
The notation $\mathcal O_A$ hides constants depending on the fixed game.
The bounds of \citet{ito2026adversarial} use the equality of the pure
maximin and Nash values in the listed cases.
The entries for \citet{maiti2025limitations} include the extensions to
nonunique equilibria discussed in Appendices~D.5 and~E.4 and report
the fixed-horizon rates.
Equilibria need not be unique unless strictness is specified.}
\label{tab:regret-comparison}
\begin{tabular}{@{}p{.27\linewidth}p{.23\linewidth}p{.25\linewidth}p{.21\linewidth}@{}}
\toprule
Work & Feedback & Games & Nash regret \\
\midrule
\citet{o2021matrix}
& Bandit + actions & Any $n\times m$
& $\widetilde{\mathcal O}(\sqrt{nmT})$ \\
\addlinespace
\citet{ito2026adversarial}
& Bandit only & $n\times m$, strict pure NE
& $\mathcal O_A(\log T)$ \\
\addlinespace
\citet{ito2026adversarial}
& Bandit + actions & $n\times m$, pure row optimum
& $\mathcal O_A(\log T)$ \\
\addlinespace
\citet{maiti2025limitations}
& Full matrix + actions & Any $n\times m$
& $\mathcal O_A(\log^2 T)$ \\
\addlinespace
\citet{maiti2025limitations}
& Bandit + actions & Any $2\times2$
& $\mathcal O_A(\log^2 T)$ \\
\midrule
\textbf{This work}
& Bandit + actions & \textbf{Any} $n\times m$
& $\mathcal O_A(\log^2 T)$ \\
\bottomrule
\end{tabular}
\end{table}

Prior work establishes several routes to low regret under different
assumptions on feedback and opponent behavior.
With bandit payoff feedback and observed opponent actions,
\citet{o2021matrix} analyze an optimistic matrix-game algorithm with
$\widetilde{\mathcal O}(\sqrt{nmT})$ Nash regret for an $n\times m$ game.
In self-play, where both players follow prescribed learning algorithms,
\citet{ito2025instance} obtain improved instance-dependent regret bounds,
including logarithmic regret when the Nash equilibrium is unique and pure.
Against arbitrary opponents, \citet{ito2026adversarial} obtain logarithmic
regret relative to the best worst-case payoff of a pure strategy when
opponent actions are observed. This benchmark can be smaller than the
Nash value, but the two coincide whenever the row player has a pure
optimal strategy, including games with a pure Nash equilibrium.
Their result therefore also gives logarithmic Nash regret for this class
of games.

Our work is most closely related to \citet{maiti2025limitations}.
They show that both playing an empirical equilibrium and using
matrix-game UCB can incur $\Omega(\sqrt T)$ Nash regret even on fixed
$2\times2$ instances.
They also develop algorithms achieving instance-dependent polylogarithmic
Nash regret, allowing nonunique equilibria.
With a noisy observation of the entire matrix at each round,
their guarantee applies to arbitrary $n\times m$ games.
Under bandit feedback, they establish this guarantee only for
$2\times2$ games, leaving the extension to arbitrary dimensions open.
The obstacle is that the learner controls only its own action:
the opponent can withhold some columns or reveal them much less often
than others, preventing uniform exploration of the matrix.
In higher dimensions, these unequal estimation errors must be handled
across several interacting strategy adjustments.

The observation of opponent actions is essential to this question.
Without it, an $\Omega(\sqrt T)$ lower bound holds even over a fixed
finite family of $2\times2$ games \citep[Remark~3]{ito2026adversarial}.
This leaves the following question:
\begin{quote}
\emph{Can a learner achieve instance-dependent polylogarithmic Nash regret
in arbitrary finite matrix games with bandit payoff feedback and observed
opponent actions?}
\end{quote}

\paragraph{Our results.}
We answer this question affirmatively, resolving the open problem of
\citet{maiti2025limitations}.
Our algorithm, Optimistic Payoff Balancing (OPB;
Algorithm~\ref{alg:main}), satisfies
\[
  R_T\le C(A)\log^2(eT)\qquad\text{for every }T\ge1
\]
for every fixed finite payoff matrix $A$, where $C(A)$ is finite and
depends only on the matrix, including its dimensions
(Theorem~\ref{thm:regret}).
The guarantee holds against adaptive opponents and allows nonunique
Nash equilibria for both players.
Moreover, OPB requires neither the horizon nor prior knowledge of
equilibrium supports or game-dependent separation parameters.
Thus, bandit payoff feedback with observed actions suffices for
polylogarithmic Nash regret in arbitrary finite dimensions.
Table~\ref{tab:regret-comparison} compares the settings and guarantees.

\paragraph{Technical ideas.}
We build on the use of opponent actions to adjust a strategy near an
estimated equilibrium \citep{maiti2025limitations}.
The challenge in higher dimensions is to handle unequal estimation
errors across columns while allowing nonunique equilibria.
We address nonuniqueness by maximizing a joint logarithmic objective
over row probabilities and column slacks. Once estimates are accurate
enough, the resulting reference gives positive probability to every row
used by an optimum and positive slack to every column whose payoff can
exceed the value at an optimum, leaving room for local adjustments.

Our central technique is to select independent payoff differences in
order of estimation accuracy and scale each coordinate's adjustment
range by its uncertainty. We also shrink estimated payoff coefficients
toward zero using confidence bounds, so a truly zero coefficient never
triggers an update. The selection order ensures that repeated visits
affecting a coordinate either bring its sample counts to the next
reconstruction threshold or create a persistent payoff imbalance.
The local updates exploit this imbalance to earn surplus reward.
Its negative regret contribution absorbs the accumulated estimation
cost, giving logarithmic regret per epoch and an overall
$\mathcal O_A(\log^2 T)$ bound.

\section{Preliminaries}\label{sec:preliminaries}

\paragraph{Notation.}
For a positive integer $k$, let $[k]=\{1,\ldots,k\}$ and
$\Delta_k=\{p\in\mathbb R_{\ge0}^k:\sum_{i=1}^k p_i=1\}$.
We write $p_i$ for the $i$th coordinate of a vector $p$ and $A_j$
for the $j$th column of a matrix $A$. All logarithms are natural.

\paragraph{Game and feedback model.}
We consider a two-player zero-sum game with a fixed unknown payoff matrix
$A\in[-1,1]^{n\times m}$. The row player has $n$ actions and maximizes
reward, while the column player has $m$ actions and minimizes it.
The entry $A_{ij}$ is the expected reward of the row player when the
players choose actions $i$ and $j$. For mixed strategies
$p\in\Delta_n$ and $q\in\Delta_m$, the expected payoff is $p^\top Aq$.
A pair $(p^\star,q^\star)\in\Delta_n\times\Delta_m$ is a Nash equilibrium
(NE) if neither player can improve its payoff by changing its own strategy:
\[
  p^\top Aq^\star
  \le (p^\star)^\top Aq^\star
  \le (p^\star)^\top Aq
  \qquad\text{for all }p\in\Delta_n,\ q\in\Delta_m.
\]
The value of the game is
\[
  v(A)=\max_{p\in\Delta_n}\min_{q\in\Delta_m}p^\top Aq
      =(p^\star)^\top Aq^\star.
\]
An equilibrium row strategy guarantees expected reward at least $v(A)$
against every opponent strategy. Nash equilibria need not be unique:
either player may have multiple equilibrium strategies.

The learner controls the row player and interacts with an adaptive
opponent under bandit payoff feedback with observed opponent actions.
At round $t$, the learner chooses a mixed strategy $p_t\in\Delta_n$
using the past observations. The opponent chooses a mixed strategy
$q_t\in\Delta_m$, possibly using the matrix $A$, the history, and $p_t$.
Conditionally on this history and the chosen strategies, the actions
$I_t\sim p_t$ and $J_t\sim q_t$ are sampled independently.
The learner then observes $(I_t,J_t,r_t)$, where $r_t\in[-1,1]$
is a noisy payoff with conditional mean $A_{I_tJ_t}$.
Specifically, writing $\mathcal F_t^-$ for the information available
after the mixed strategies are chosen and before the actions are drawn,
the reward is generated according to
\[
  \mathbb E[r_t\mid\mathcal F_t^-,I_t,J_t]=A_{I_tJ_t}.
\]
This feedback model allows the opponent to respond to the learner's
mixed strategy and allows the reward noise to depend on the history.

\paragraph{Nash regret.}
We measure performance by the expected Nash regret
\[
  R_T
  =Tv(A)-\mathbb E\!\left[\sum_{t=1}^T r_t\right]
  =\mathbb E\!\left[\sum_{t=1}^T
       \bigl(v(A)-p_t^\top Aq_t\bigr)\right].
\]
The expectation includes the randomness of the learner, the opponent,
and the rewards.

Under \emph{uninformed} bandit feedback, where the learner observes
the reward but not the opponent's action, an $\Omega(\sqrt{T})$ Nash
regret lower bound holds even over a fixed finite family of $2\times2$ games
\citep[Remark~3]{ito2026adversarial}.
We therefore ask whether observing the opponent's action, as in the
\emph{informed} feedback model above, allows instance-dependent
polylogarithmic Nash regret.

Our goal is a single algorithm that does not require the matrix $A$
or the time horizon and achieves, for every fixed $A$,
\[
  R_T\le C(A)\operatorname{polylog}(eT)
  \qquad\text{for every integer }T\ge1,
\]
where $C(A)$ is finite and depends only on the payoff matrix, including
its dimensions.

\section{Achieving Polylogarithmic Nash Regret}\label{sec:polylog}

Section~\ref{sec:design} develops the algorithm step by step, explaining
the challenges that motivate each component.
Section~\ref{sec:guarantee} states the regret guarantee, discusses its
implications, and gives a proof sketch.

\subsection{Algorithm Design}\label{sec:design}

Our algorithm, Optimistic Payoff Balancing (OPB), combines payoff
estimation with local strategy updates.
It periodically estimates the game and constructs a reference strategy
near its optimal set. Between these reconstructions, it uses each observed
opponent action to adjust the strategy. These adjustments correct
estimation bias and exploit persistent column imbalance to earn surplus
over the game value.

The main challenge is that the opponent controls which columns are
sampled, so different payoff estimates can have very different accuracies.
Holding an estimated equilibrium fixed can then incur the same
estimation bias repeatedly. We address this by expressing both payoff
uncertainty and strategy adjustments in a common set of payoff difference
coordinates. A further challenge is that equilibria can be nonunique:
we therefore construct an interior reference and select independent
payoff differences to define the local updates.

We first describe a run with a planned length $H$. The learner maintains
the number of observations $C_{ij}$ and empirical mean $\widehat A_{ij}$
of each entry. It keeps each local model fixed during an \emph{epoch},
while updating its strategy after each observed opponent action.
Each reconstruction incorporates the new observations and restarts the
local update. We now describe how to construct the model, choose its
coordinates, and update the strategy, before assembling the complete
procedure and removing the need to know the horizon.

\paragraph{Learning entries through optimistic completion.}
The learner cannot force the opponent to reveal a particular column.
We call an entry \emph{acquired} once it has been observed at least $h$ times.
We therefore assign the upper payoff bound $1$ to entries with fewer
than $h$ observations, and use empirical means for the remaining entries.
At the start of an epoch, we form the acquisition mask
$M=\{(i,j):C_{ij}\ge h\}$ of acquired entries and set
\[
  \widehat A^M_{ij}=
  \begin{cases}
    \widehat A_{ij},&(i,j)\in M,\\
    1,&(i,j)\notin M.
  \end{cases}
\]
This completion makes an unresolved entry attractive until enough samples
are collected. Its cost is controlled directly by those samples: each
entry is observed at most $h$ times before its empirical mean takes over,
so the total expected payoff overstatement is at most $2nmh$.
We can thus analyze learning in the completed game while paying a finite
acquisition cost for the original one.

Let $\varepsilon$ be the desired accuracy of acquired entries and let
$\tau$ distinguish features of the optimal set from estimation error.
We use
\[
  \ell=\log\bigl(64nm(H+1)^4\bigr),\quad
  \varepsilon=\ell^{-1/2},\quad
  h=\lceil2\ell^2\rceil,\quad
  \tau=\min\{\sqrt\varepsilon,1/(2n)\}.
\]
An entry with $c$ observations has confidence radius $\sqrt{2\ell/c}$.
Since $\sqrt{2\ell/h}\le\varepsilon$, $h$ observations give accuracy
at most $\varepsilon$, and the
acquisition cost is of order $\ell^2$. The larger threshold $\tau$
separates features that vanish with estimation error from fixed positive
features of the game. We reconstruct the model initially and whenever an
entry count reaches one of $h,2h,4h,\ldots$. Between reconstructions,
the confidence widths of acquired entries change by at most a constant
factor. All quantities defining the model below are frozen within an epoch.

\paragraph{Finding an interior reference when equilibria are nonunique.}
For the current mask $M$, let $A^M$ be the true completed game, with
$A^M_{ij}=A_{ij}$ on $M$ and $A^M_{ij}=1$ elsewhere, and write
$v_M=v(A^M)$. The optimal set consists of the strategies $x\in\Delta_n$
satisfying $(A^M)^\top x\ge v_M\mathbf1$.
We call column $j$ \emph{binding} at a strategy $x$
if its slack $x^\top A^M_j-v_M$ is zero.
Two optimal row strategies $x,y$ can satisfy
\[
  x_i=0<y_i
  \qquad\text{or}\qquad
  x^\top A^M_j=v_M<y^\top A^M_j.
\]
In the first case, choosing $x$ excludes a row that another optimum uses.
In the second, column $j$ is binding at $x$ but has positive slack at $y$.
An empirical equilibrium can lie near either type of boundary.
We therefore seek a reference in the optimal set whose row probability
is positive whenever $x_i>0$ is possible at an optimum, and whose
column slack is positive whenever $x^\top A^M_j-v_M>0$ is possible.
These positive probabilities and slacks provide room for local adjustments.

We write $\widehat v=v(\widehat A^M)$ and define the near-optimal region
\[
  Z=\{x\in\Delta_n:(\widehat A^M)^\top x
                 \ge(\widehat v-2\varepsilon)\mathbf1\}.
\]
When the empirical entry errors satisfy the confidence bounds above,
$|\widehat v-v_M|\le\varepsilon$. Every optimal strategy $x$ of $A^M$
therefore satisfies
\[
  x^\top\widehat A^M_j
  \ge x^\top A^M_j-\varepsilon
  \ge v_M-\varepsilon
  \ge\widehat v-2\varepsilon
  \qquad(j\in[m]),
\]
so $Z$ contains the entire true optimal set.
We define the nonnegative empirical slack
$s_j(x)=x^\top\widehat A^M_j-\widehat v+2\varepsilon$ on $Z$
and choose the joint center
\begin{equation}\label{eq:center}
  p^c\in\mathop{\rm argmax}_{x\in Z}
  \left\{\sum_{i=1}^n\log(\varepsilon+x_i)
       +\sum_{j=1}^m\log(\varepsilon+s_j(x))\right\}.
\end{equation}
The two sums favor positive row probabilities and positive slacks
together. The shift by $\varepsilon$ accommodates features that are
zero throughout the optimal set. We then select the rows and binding
columns indicated by the center:
\begin{equation}\label{eq:face}
  I=\{i:p_i^c>\tau\},\qquad
  J=\{j:s_j(p^c)\le\tau\},\qquad
  \bar p_i=\frac{p_i^c}{\sum_{k\in I}p_k^c}\quad(i\in I).
\end{equation}
Since $\max_i p_i^c\ge1/n>\tau$, the normalization defining $\bar p$
is always well defined.
For each fixed game, sufficiently accurate estimates identify the union
of optimal row supports and the columns binding at every optimum.
We refer to these columns as the game's binding columns, and to the
remaining columns as nonbinding.
The reference $\bar p$ then assigns a fixed positive probability to every
retained row and has positive surplus against every nonbinding column.
Both properties survive small local adjustments. Removing rows outside
$I$ is also important: assigning even a small probability to a
strictly suboptimal row can accumulate loss throughout the run.

\paragraph{Selecting independent payoff equalities.}
At an optimal strategy $p$, every true binding column gives payoff $v_M$.
Thus, for any two such columns $j$ and $j_0$,
\[
  p^\top A^M_j=p^\top A^M_{j_0}=v_M,
  \qquad\text{so}\qquad
  p^\top(A^M_j-A^M_{j_0})=0.
\]
Subtracting the reference column removes the unknown value $v_M$.
The algorithm uses the empirical equations
$p^\top(\widehat A^M_j-\widehat A^M_{j_0})=0$ for columns in $J$.
Some equations can follow from others, so we select a subset whose
left-hand sides can be adjusted independently while preserving total
probability. The next step uses these quantities to specify how the
strategy should change.

The estimates for different columns can have very different accuracies,
and estimation noise can make redundant equations appear independent.
We therefore process columns from most to least accurately estimated.
We add a column's equation only when a confidence test certifies that
it is independent of the equations already selected, even after allowing
for estimation error.

Let $d=|I|$, $\widehat B=\widehat A^M[I,:]$, and
$P_0=\mathrm{Id}_d-\mathbf1\mathbf1^\top/d$, the projection onto
the subspace $\{u\in\mathbb R^d:\mathbf1^\top u=0\}$ of changes
that preserve total probability. We call the entries $(i,j)\notin M$
\emph{artificial entries}: both $\widehat A^M$ and $A^M$ assign them
the fixed value $1$. They therefore have no estimation error in the
completed game. To reflect this, we define effective counts by
\[
  \widetilde C_{ij}=
  \begin{cases}C_{ij},&(i,j)\in M,\\+\infty,&(i,j)\notin M,\end{cases}
  \qquad c_j=\min_{i\in I}\widetilde C_{ij}.
\]
If $J$ is nonempty, we choose a reference column
$j_0\in\mathop{\rm argmax}_{j\in J}c_j$ and process the other columns
of $J$ in decreasing count order, breaking ties by index. Starting with
an empty list of columns, we tentatively append each candidate $j$.
For the resulting trial list $j_1,\ldots,j_k$, we define
\[
  \widehat D'
  =\bigl[\widehat B_{j_1}-\widehat B_{j_0},\ldots,
         \widehat B_{j_k}-\widehat B_{j_0}\bigr]
  \in\mathbb R^{d\times k}.
\]
Let $e'_l=\sqrt{2\ell/c_{j_l}}$ for $l\in[k]$ be the corresponding
error scales. If
$G'=P_0\widehat D'$ has full column rank, we compute
\begin{equation}\label{eq:certificate}
  R'=G'((G')^\top G')^{-1},\qquad
  \chi'=2\sum_l e'_l\|R'_{\cdot l}\|_1,
\end{equation}
and accept the new equation precisely when $\chi'\le1/2$.
We reject trials without full column rank.

The identities $(\widehat D')^\top R'=\mathrm{Id}_k$ and
$\mathbf1^\top R'=0$ show how $R'$ adjusts the selected payoff differences:
adding $R'u$ to a strategy changes these differences by $u\in\mathbb R^k$
without changing total probability. The quantity $\chi'$ measures how much
entry uncertainty is amplified by this adjustment. We call the test $\chi'\le1/2$ the independence
certificate: it ensures that a perturbation within
the confidence bounds cannot destroy independence. It rejects directions
created entirely by noise and eventually accepts every truly independent
equation. A column with $c_j=\infty$ is the all-ones vector on $I$;
its reference is also all-ones, so its zero difference is rejected.

We denote the accepted columns by $j_1,\ldots,j_b$, where $b\le d-1$,
and set
\[
  \begin{gathered}
    \widehat D_l=\widehat B_{j_l}-\widehat B_{j_0},\qquad
    \widehat D=[\widehat D_1,\ldots,\widehat D_b],\\
    e_l=\sqrt{2\ell/c_{j_l}},\qquad
    E=\operatorname{diag}(e_1,\ldots,e_b),\\
    \widehat R=(P_0\widehat D)(\widehat D^\top P_0\widehat D)^{-1}.
  \end{gathered}
\]
The selected errors satisfy $0<e_1\le\cdots\le e_b\le\varepsilon$.
If $J$ is empty or no equation is accepted, we set $b=0$ and play $\bar p$
throughout the epoch. The following coordinate operations apply when $b>0$.

\paragraph{Using uncertainty to set the adjustment range.}
Once the optimal row support and binding columns are identified,
a true optimum $p^\star$ supported on $I$, viewed as a vector in
$\mathbb R^I$, satisfies $|\widehat D_l^\top p^\star|\le2e_l$,
since the corresponding true
payoff difference is zero. We allow residuals
$\widehat D_l^\top p=4e_lz_l$ with $z_l\in[-1,1]$, so the optimum's
normalized residual obeys
\[
  |z_l^\star|
  =\frac{|\widehat D_l^\top p^\star|}{4e_l}
  \le\frac12.
\]
The factor four thus leaves at least $1/2$ of the coordinate range
on either side of $z_l^\star$ for responding to the opponent's play.

Under nonuniqueness, these residuals need not determine a unique strategy.
We choose the point nearest to the reference $\bar p$ with the prescribed
residuals and total mass one. Since
$\widehat D^\top\widehat R=\mathrm{Id}_b$ and
$\mathbf1^\top\widehat R=0$, it is
\[
  \widehat x=\bar p-\widehat R\widehat D^\top\bar p,
  \qquad x(z)=\widehat x+4\widehat R E z.
\]
The strategy $x(z)$ has the prescribed payoff differences and total mass one:
\[
  \widehat D^\top x(z)=4Ez,
  \qquad \mathbf1^\top x(z)=1.
\]
For sufficiently accurate estimates, every $z\in[-1,1]^b$ also gives
$x_i(z)>0$ for all $i\in I$, and the payoff against every nonbinding
column remains above $v_M$. With less accurate estimates, however,
$x(z)$ may have negative coordinates. We therefore play its Euclidean
projection onto
\[
  \Delta_I=\{x\in\mathbb R_{\ge0}^I:\textstyle\sum_{i\in I}x_i=1\}.
\]
Here and below, $\Pi_S$ denotes Euclidean projection onto $S$.

\paragraph{Shrinking payoff coefficients using confidence bounds.}
The empirical payoff against column $j$ is linear in these coordinates:
\[
  x(z)^\top\widehat B_j
  =\widehat x^\top\widehat B_j+4\sum_l e_l\widehat\alpha_{lj}z_l,
  \qquad \widehat\alpha_j=\widehat R^\top\widehat B_j.
\]
Thus, increasing $z_l$ by an amount $\Delta z_l$ changes the predicted
payoff against column $j$ by $4e_l\widehat\alpha_{lj}\Delta z_l$.
An error in $\widehat\alpha_{lj}$ can therefore misdirect the strategy
update whenever the opponent plays column $j$. Since the coefficients
are fixed within an epoch, the same error can influence many successive
updates. We use the entry confidence bounds to quantify uncertainty in
each coefficient and shrink its estimate toward zero before updating.

For $j\in J$, we set $\delta_j=\sqrt{2\ell/c_j}$, with
$\delta_j=0$ when $c_j=\infty$, and define
\begin{equation}\label{eq:conservative}
  \begin{gathered}
    \nu_l=\|\widehat R_{\cdot l}\|_1,\qquad
    b_{lj}=2\nu_l\left(\delta_j+2\sum_k e_k|\widehat\alpha_{kj}|\right),\\
    \widetilde\alpha_{lj}
    =\operatorname{sgn}(\widehat\alpha_{lj})
      (|\widehat\alpha_{lj}|-b_{lj})_+.
  \end{gathered}
\end{equation}
Here $(a)_+=\max\{a,0\}$ and $\operatorname{sgn}(0)=0$.
The term $\delta_j$ bounds the estimation error in $\widehat B_j$.
Each selected difference $\widehat D_k$ also has error at most $2e_k$
in each entry. In the linear combination
$\sum_k\widehat\alpha_{kj}\widehat D_k$, these errors are multiplied by
$|\widehat\alpha_{kj}|$, giving the bound
$2\sum_k e_k|\widehat\alpha_{kj}|$.
The factor $\nu_l$ translates these payoff errors into uncertainty
in the $l$th coefficient.
To state what this interval estimates, we write $B=A^M[I,:]$,
$D_k=B_{j_k}-B_{j_0}$, and $D=[D_1,\ldots,D_b]$.
For sufficiently accurate estimates, $I$ consists exactly of the rows
used by at least one optimal strategy, and $J$ consists exactly of the
columns binding at every optimum. The selected differences
$D_1,\ldots,D_b$ then express each column $B_j$, $j\in J$, as the
constant payoff $v_M\mathbf1$ plus a linear combination.
The coefficients $\alpha_j$ satisfy
\[
  B_j=v_M\mathbf1+D\alpha_j,
  \qquad
  |\widehat\alpha_{lj}-\alpha_{lj}|\le b_{lj}.
\]
We choose $\widetilde\alpha_{lj}$ as the point nearest zero in
$[\widehat\alpha_{lj}-b_{lj},\widehat\alpha_{lj}+b_{lj}]$.
If $|\widehat\alpha_{lj}|\le b_{lj}$, this gives
$\widetilde\alpha_{lj}=0$; otherwise, we keep the estimated sign and
reduce its magnitude by $b_{lj}$. In particular,
$\alpha_{lj}=0$ implies $\widetilde\alpha_{lj}=0$.
For $j\notin J$, we set $\widetilde\alpha_j=0$; the local neighborhood
already preserves positive surplus against these columns.

At the first round $t_0$ of each epoch, we initialize $z_{t_0}=0$.
After observing the opponent's action $J_t$, we perform projected gradient
ascent in the coordinates $z$:
\begin{equation}\label{eq:local-update}
  p_t=\Pi_{\Delta_I}\bigl(\widehat x+4\widehat R E z_t\bigr),
  \qquad
  z_{t+1}=\Pi_{[-1,1]^b}\bigl(z_t+4E\widetilde\alpha_{J_t}\bigr).
\end{equation}
We extend $p_t$ by zero outside $I$. This is the standard projected gradient
update \citep{zinkevich2003online}, applied to the conservative linear
payoff model. We use the observed opponent action $J_t$ to choose
$\widetilde\alpha_{J_t}$ for the strategy update from $z_t$ to $z_{t+1}$.
We use the reward $r_t$ to update the empirical mean
$\widehat A_{I_tJ_t}$ of the sampled matrix entry.
Both the adjustment range and the conservative gradient thus come from
the same entry confidence bounds.

\paragraph{Complete procedure.}
Algorithm~\ref{alg:main} assembles the four steps into epochs.
At the start of each epoch, we complete the empirical matrix and compute
the reference $\bar p$, select independent payoff equalities, and construct
the strategy map and corrected coefficients $\widetilde\alpha_j$.
We then keep this model fixed: each observed column $J_t$ determines
the update of $z$, while the reward $r_t$ updates the sampled entry's
count and empirical mean. When a count reaches $h,2h,4h,\ldots$,
we rebuild the model and reset $z=0$ before the next round, keeping
all accumulated observations.

To remove the need to know the horizon, we apply the doubling trick
to the logarithm of the planned run length:
\[
  H_0=2,\qquad H_{r+1}=H_r^2,
  \qquad\text{so}\qquad H_r=2^{2^r}.
\]
After $H_r$ rounds, we start a fresh run with length $H_{r+1}$,
resetting its counts and empirical means.
Since $\log H_{r+1}=2\log H_r$, the polylogarithmic regret bounds
of successive runs sum geometrically.

\begin{algobox}{Optimistic Payoff Balancing (OPB)}\label{alg:main}
\textbf{Input:} Numbers of actions $n,m$.
For runs $r=0,1,\ldots$, reset counts $C_{ij}$ and means
$\widehat A_{ij}$ to zero and set
\[
 H=2^{2^r},\quad \ell=\log(64nm(H+1)^4),\quad
 \varepsilon=\ell^{-1/2},\quad h=\lceil2\ell^2\rceil,\quad
 \tau=\min\{\sqrt\varepsilon,1/(2n)\}.
\]
Repeat the following epochs until the run has played $H$ rounds.
\begin{enumerate}
\setlength{\itemsep}{3pt}\setlength{\parsep}{0pt}
\item \textbf{Construct the reference.}
Set $M=\{(i,j):C_{ij}\ge h\}$ and fill $\widehat A^M$ with
$\widehat A_{ij}$ on $M$ and $1$ elsewhere.
Let $\widehat v$ be its game value,
$s_j(x)=x^\top\widehat A^M_j-\widehat v+2\varepsilon$, and
$Z=\{x\in\Delta_n:s_j(x)\ge0\ \forall j\}$. Compute
\[
 p^c\in\mathop{\rm argmax}_{x\in Z}
 \left\{\sum_i\log(\varepsilon+x_i)+\sum_j\log(\varepsilon+s_j(x))\right\}.
\]
Set $I=\{i:p_i^c>\tau\}$, $J=\{j:s_j(p^c)\le\tau\}$,
$\bar p_i=p_i^c/\sum_{k\in I}p_k^c$ for $i\in I$.

\item \textbf{Select independent equations.}
Set $d=|I|$, $\widehat B=\widehat A^M[I,:]$,
$P_0=\mathrm{Id}_d-\mathbf1\mathbf1^\top/d$, and
$c_j=\min\{C_{ij}:i\in I,\ (i,j)\in M\}$, with $\min\varnothing=\infty$.
For nonempty $J$, choose $j_0\in\arg\max_{j\in J}c_j$ and scan the
remaining columns in decreasing $c_j$, breaking ties by index.
Starting from an empty list, append each candidate tentatively.
For the trial list $j'_1,\ldots,j'_k$, form
\[
 \widehat D'=[\widehat B_{j'_l}-\widehat B_{j_0}]_{l=1}^k,\qquad
 G'=P_0\widehat D'.
\]
Reject the candidate if $\operatorname{rank}(G')<k$; otherwise compute
\[
 R'=G'((G')^\top G')^{-1},\qquad
 \chi'=2\sum_{l=1}^k\sqrt{2\ell/c_{j'_l}}\,\|R'_{\cdot l}\|_1,
\]
and accept it iff $\chi'\le1/2$.
Denote the accepted list by $j_1,\ldots,j_b$; set $b=0$ if $J=\varnothing$.

\item \textbf{Build the local update} (if $b>0$).
Set $\widehat D=[\widehat B_{j_l}-\widehat B_{j_0}]_{l=1}^b$ and
$e_l=\sqrt{2\ell/c_{j_l}}$, $E=\operatorname{diag}(e_1,\ldots,e_b)$.
Compute
\[
 \widehat R=(P_0\widehat D)(\widehat D^\top P_0\widehat D)^{-1},\qquad
 \widehat x=\bar p-\widehat R\widehat D^\top\bar p.
\]
For $j\in J$ and $l\in[b]$, compute
\[
 \begin{gathered}
 \widehat\alpha_j=\widehat R^\top\widehat B_j,\qquad
 b_{lj}=2\|\widehat R_{\cdot l}\|_1
 \left(\sqrt{2\ell/c_j}+2\sum_k e_k|\widehat\alpha_{kj}|\right),\\
 \widetilde\alpha_{lj}
 =\operatorname{sgn}(\widehat\alpha_{lj})(|\widehat\alpha_{lj}|-b_{lj})_+.
 \end{gathered}
\]
Set $\widetilde\alpha_j=0$ for $j\notin J$ and initialize $z=0$.

\item \textbf{Play until the next epoch.}
Keep the model fixed. If $b=0$, play $\bar p$; otherwise play
$p_I=\Pi_{\Delta_I}(\widehat x+4\widehat R E z)$.
Set $p_i=0$ outside $I$, draw $I_t\sim p$, and observe $J_t,r_t$.
For $b>0$, update
\[
 z\leftarrow\Pi_{[-1,1]^b}(z+4E\widetilde\alpha_{J_t}).
\]
Increase $C_{I_tJ_t}$ by one and update its empirical mean using $r_t$.
When this count reaches $h,2h,4h,\ldots$, start a new epoch on the next
round. End the run after $H$ rounds.
\end{enumerate}
Here $\Delta_I$ is the probability simplex on $I$, $\Pi$ is Euclidean
projection, $(a)_+=\max\{a,0\}$, and $\sqrt{2\ell/\infty}=0$.
\end{algobox}

\subsection{Regret Guarantee}\label{sec:guarantee}

In the following theorem, we prove that OPB achieves polylogarithmic Nash
regret for every fixed finite matrix game, including games with nonunique
equilibria.

\begin{theorem}[Polylogarithmic Nash regret]\label{thm:regret}
For every $n,m\ge1$ and every $A\in[-1,1]^{n\times m}$, there is a
finite constant $C(A)$ such that Algorithm~\ref{alg:main} satisfies
\[
  R_T\le C(A)\log^2(eT)\qquad\text{for every integer }T\ge1.
\]
The same constant applies to every opponent and reward process in the
feedback model of Section~\ref{sec:preliminaries}.
\end{theorem}

Theorem~\ref{thm:regret} resolves the open problem posed by
\citet{maiti2025limitations}: whether instance-dependent polylogarithmic
Nash regret is achievable in arbitrary $n\times m$ matrix games under
bandit payoff feedback with observed opponent actions.
Their results establish polylogarithmic Nash regret for general matrices
with full-matrix feedback, but cover only $2\times2$ games under bandit
feedback. OPB closes this gap by achieving
$\mathcal{O}_A(\log^2 T)$ Nash regret in arbitrary finite dimensions
against adaptive opponents.

The guarantee allows nonunique Nash equilibria for both players.
Moreover, OPB needs no prior knowledge of the equilibrium supports or
game-dependent separation parameters.

\paragraph{Proof sketch.}
We analyze a run with planned length $H$ through the true completed
games $A^M$. Optimistic completion costs at most $2nmh$ in expected
regret. On a simultaneous confidence event, for sufficiently large $H$,
the center and certificate identify the optimal row supports, binding
columns, and independent payoff differences. Every local strategy then
assigns positive probability bounded away from zero to each selected row
and earns surplus against nonbinding columns.

Fix an epoch prefix and write $B=A^M[I,:]$.
For the true differences $D_l=B_{j_l}-B_{j_0}$, each binding column
has the representation $B_j=v(B)\mathbf1+D\alpha_j$.
Let $N_l$ count visits with $J_t\in J$ and $\alpha_{lJ_t}\ne0$, and set
\[
  d_l=\sum_{t:J_t\in J}\alpha_{lJ_t},\qquad
  V=\sum_l e_l^2N_l,\qquad U=\sum_l e_l|d_l|.
\]
Here $V$ measures estimation cost and $U$ measures the available payoff
advantage. The count ordering ensures that columns with
$c_j>q_l:=c_{j_l}$ share the coefficient $\alpha_{l j_0}$, while
the remaining binding columns receive only $\mathcal O_A(q_l)$ visits
before reconstruction. Thus
\[
  N_l\le K_A(q_l+|d_l|),\qquad
  V\le K_A\ell+K_A\varepsilon U,
\]
where $K_A$ depends only on $A$ and may increase between displays.

For the regret bound, we construct a true optimum $x(z^*)$ with
$\|z^*\|_\infty\le1/2$ and adjust its coordinates to obtain a moving
comparison point $u(t)$ that corrects for the empirical payoff
differences. We show that $\|u(t)\|_\infty\le3/4$ and that each
coordinate has total variation at most $K_AV$.
The implication $\alpha_{lj}=0\Rightarrow\widetilde\alpha_{lj}=0$
is essential here: coordinate $l$ changes only on its $N_l$ counted
visits. We compare projected gradient ascent with
$u_l(t)+\frac14\operatorname{sgn}(d_l)\in[-1,1]$.
The shift contributes the advantage $U$, while comparator variation and
coefficient errors cost at most $K_A(\ell+V)$. Consequently,
\[
  \mathcal R_{\rm epoch}
  :=\sum_t\bigl(v(B)-p_t^\top B_{J_t}\bigr)
  \le K_A\ell+K_AV-U
  \le K_A\ell+(K_A\varepsilon-1)U,
\]
where the sum is over the epoch prefix and $p_t$ is restricted to $I$.
For sufficiently large $H$, $K_A\varepsilon\le1$, giving
$\mathcal R_{\rm epoch}\le K_A\ell$.
When $b=0$, the epoch deficit is already nonpositive.

There are $\mathcal O_{n,m}(1+\log H)$ epochs.
Since $\ell=\mathcal O_{n,m}(\log(eH))$ and $h=\mathcal O(\ell^2)$,
adding acquisition and confidence failure costs gives
$\mathcal O_A(\log^2(eH))$ regret for every run prefix.
Smaller horizons are absorbed into the instance-dependent constant.
Finally, $\log H_r=2^r\log2$ grows geometrically, so summing the
conditional run bounds gives $\mathcal O_A(\log^2(eT))$ regret.
Appendix~\ref{app:analysis} provides the complete proof.

\section{Related Work}\label{sec:related-work}

\paragraph{Regret minimization in matrix games.}
Regret minimization in repeated games has a long history, from
approachability and consistent play to adaptive algorithms and regret
matching
\citep{blackwell1956analog,hannan1957approximation,freund1999adaptive,foster1999regret,hart2000simple}.
Under bandit feedback, no-regret algorithms compete with the best fixed
action using only the rewards of the actions played
\citep{auer2002nonstochastic,neu2015explore}.
A related direction seeks algorithms that adapt to favorable problem
structure while preserving worst-case guarantees, leading to
instance-dependent bounds in bandits and self-play games
\citep{zimmert2021tsallis,ito2025instance}.
For learning against arbitrary opponents, recent work exploits the fixed
payoff matrix and studies regret relative to the Nash value or the pure
maximin value \citep{o2021matrix,maiti2025limitations,ito2026adversarial}.
These benchmarks distinguish the payoff guaranteed by a mixed strategy
from that guaranteed by a pure action.
Our result addresses the Nash benchmark in arbitrary finite matrix games;
Section~\ref{sec:introduction} and Table~\ref{tab:regret-comparison}
compare the feedback models and guarantees.

\paragraph{Equilibrium learning with bandit feedback.}
Another line of work studies equilibrium learning when players observe
only their own actions and realized payoffs.
Early approaches include reinforcement learning and regret testing,
with asymptotic guarantees under different equilibrium concepts and
game assumptions
\citep{hart2001reinforcement,leslie2005individual,foster2006regret,germano2007global}.
Subsequent work develops finite-sample guarantees for learning from
payoff observations in matrix, Markov, polymatrix, and monotone games
\citep{cai2023uncoupled,chen2024decentralized,faizal2024finite,dong2026uncoupled}.
One direction studies last-iterate convergence, including achievable
rates, computational efficiency, and the role of observed opponent actions
\citep{cai2026average,fiegel2026harder,maiti2026efficient,fiegel2026optimal,hait2026near}.
These works study equilibrium learning when all players follow specified
algorithms, whereas we study the cumulative payoff of one learner facing
an arbitrary opponent.

\section{Conclusion}\label{sec:conclusion}

We study Nash regret minimization in unknown finite matrix games with
bandit payoff feedback and observed opponent actions.
Our algorithm, OPB, achieves instance-dependent
$\mathcal{O}_A(\log^2 T)$ Nash regret against arbitrary adaptive opponents,
including games with nonunique Nash equilibria.
This resolves the open problem of \citet{maiti2025limitations} on
polylogarithmic Nash regret under bandit feedback in arbitrary dimensions.
Our construction jointly centers row probabilities and column slacks
to leave room for local strategy adjustments even when equilibria are
nonunique.
By selecting independent payoff differences in order of estimation
accuracy and scaling the adjustments by their uncertainty, OPB exploits
imbalances in the opponent's play to earn surplus reward that absorbs
the accumulated estimation cost.
Together, these ideas establish that observing opponent actions suffices
for polylogarithmic Nash regret in general finite matrix games.

\subsection*{AI use statement}
We use GPT-6 Astra to polish the writing, assist with calculations
in the proofs, and check their correctness. We review all AI-assisted
content and take full responsibility for the final content of this paper.

\bibliography{ref}

@article{cai2023uncoupled,
  title={Uncoupled and convergent learning in two-player zero-sum {Markov} games with bandit feedback},
  author={Cai, Yang and Luo, Haipeng and Wei, Chen-Yu and Zheng, Weiqiang},
  journal={Advances in Neural Information Processing Systems},
  volume={36},
  pages={36364--36406},
  year={2023}
}

@article{cai2026average,
  title={From average-iterate to last-iterate convergence in games: A reduction and its applications},
  author={Cai, Yang and Luo, Haipeng and Wei, Chen-Yu and Zheng, Weiqiang},
  journal={Advances in Neural Information Processing Systems},
  volume={38},
  pages={46937--46967},
  year={2025}
}

@article{chen2024decentralized,
  title={Decentralized Best-Response-Based Learning in Two-Player Zero-Sum Stochastic Games: A Finite-Sample Analysis},
  author={Chen, Zaiwei and Zhang, Kaiqing and Mazumdar, Eric and Ozdaglar, Asuman and Wierman, Adam},
  journal={arXiv preprint arXiv:2409.01447},
  year={2024}
}

@article{dong2026uncoupled,
  title={Uncoupled and convergent learning in monotone games under bandit feedback},
  author={Dong, Jing and Wang, Baoxiang and Yu, Yaoliang},
  journal={Advances in Neural Information Processing Systems},
  volume={38},
  pages={151665--151683},
  year={2025}
}

@article{faizal2024finite,
  title={Finite-Sample Guarantees for Learning Dynamics in Zero-Sum Polymatrix Games},
  author={Faizal, Fathima Zarin and Ozdaglar, Asuman and Wainwright, Martin J},
  journal={arXiv preprint arXiv:2407.20128},
  year={2024}
}

@inproceedings{fiegel2026harder,
  title={The Harder Path: Last Iterate Convergence for Uncoupled Learning in Zero-Sum Games with Bandit Feedback},
  author={Fiegel, C{\^o}me and Menard, Pierre and Kozuno, Tadashi and Valko, Michal and Perchet, Vianney},
  booktitle={Proceedings of the 42nd International Conference on Machine Learning},
  pages={17131--17152},
  volume={267},
  series={Proceedings of Machine Learning Research},
  publisher={PMLR},
  year={2025},
  url={https://proceedings.mlr.press/v267/fiegel25a.html}
}

@article{fiegel2026optimal,
  title={Optimal last-iterate convergence in matrix games with bandit feedback using the log-barrier},
  author={Fiegel, C{\^o}me and Menard, Pierre and Kozuno, Tadashi and Valko, Michal and Perchet, Vianney},
  journal={arXiv preprint arXiv:2604.15242},
  year={2026}
}

@article{freund1999adaptive,
  title={Adaptive game playing using multiplicative weights},
  author={Freund, Yoav and Schapire, Robert E},
  journal={Games and Economic Behavior},
  volume={29},
  number={1-2},
  pages={79--103},
  year={1999},
  publisher={Elsevier}
}

@article{hait2026near,
  title={Near-Optimal Last-Iterate Convergence for Zero-Sum Games with Bandit Feedback and Opponent Actions},
  author={Hait, Soumita and Li, Ping and Luo, Haipeng and Zhang, Mengxiao},
  journal={arXiv preprint arXiv:2605.09363},
  year={2026}
}

@inproceedings{ito2025instance,
  title={Instance-Dependent Regret Bounds for Learning Two-Player Zero-Sum Games with Bandit Feedback},
  author={Ito, Shinji and Luo, Haipeng and Tsuchiya, Taira and Wu, Yue},
  booktitle={Proceedings of Thirty Eighth Conference on Learning Theory},
  pages={2858--2892},
  volume={291},
  series={Proceedings of Machine Learning Research},
  publisher={PMLR},
  year={2025},
  url={https://proceedings.mlr.press/v291/ito25a.html}
}

@inproceedings{maiti2026efficient,
  title={Efficient Uncoupled Learning Dynamics with {$\tilde{O}(T^{-1/4})$} Last-Iterate Convergence in Bilinear Saddle-Point Problems over Convex Sets under Bandit Feedback},
  author={Maiti, Arnab and Zhang, Claire Jie and Jamieson, Kevin and Morgenstern, Jamie Heather and Panageas, Ioannis and Ratliff, Lillian J.},
  booktitle={Proceedings of The 29th International Conference on Artificial Intelligence and Statistics},
  pages={2431--2439},
  year={2026},
  volume={300},
  series={Proceedings of Machine Learning Research},
  publisher={PMLR}
}

@article{neu2015explore,
  title={Explore no more: Improved high-probability regret bounds for non-stochastic bandits},
  author={Neu, Gergely},
  journal={Advances in Neural Information Processing Systems},
  volume={28},
  year={2015}
}

@inproceedings{o2021matrix,
  title={Matrix games with bandit feedback},
  author={O'Donoghue, Brendan and Lattimore, Tor and Osband, Ian},
  booktitle={Uncertainty in Artificial Intelligence},
  pages={279--289},
  year={2021},
  organization={PMLR}
}

@article{ito2026adversarial,
  title={Adversarial Learning in Games with Bandit Feedback: Logarithmic Pure-Strategy Maximin Regret},
  author={Ito, Shinji and Luo, Haipeng and Maiti, Arnab and Tsuchiya, Taira and Wu, Yue},
  journal={arXiv preprint arXiv:2602.06348},
  year={2026},
  url={https://arxiv.org/abs/2602.06348v1}
}

@article{maiti2025limitations,
  title={On the Limitations and Possibilities of {Nash} Regret Minimization in Zero-Sum Matrix Games under Noisy Feedback},
  author={Maiti, Arnab and Jamieson, Kevin and Ratliff, Lillian J.},
  journal={arXiv preprint arXiv:2306.13233},
  year={2025},
  url={https://arxiv.org/abs/2306.13233v3}
}

@techreport{zinkevich2003online,
  title={Online Convex Programming and Generalized Infinitesimal Gradient Ascent},
  author={Zinkevich, Martin},
  institution={Carnegie Mellon University},
  number={CMU-CS-03-110},
  year={2003},
  url={https://www.cs.cmu.edu/~maz/publications/techconvex.pdf}
}

@article{hoffman1952approximate,
  title={On Approximate Solutions of Systems of Linear Inequalities},
  author={Hoffman, Alan J.},
  journal={Journal of Research of the National Bureau of Standards},
  volume={49},
  number={4},
  pages={263--265},
  year={1952},
  doi={10.6028/jres.049.027}
}

@article{auer2002nonstochastic,
  title={The Nonstochastic Multiarmed Bandit Problem},
  author={Auer, Peter and Cesa-Bianchi, Nicol{\`o} and Freund, Yoav and Schapire, Robert E.},
  journal={SIAM Journal on Computing},
  volume={32},
  number={1},
  pages={48--77},
  year={2002},
  doi={10.1137/S0097539701398375}
}

@article{zimmert2021tsallis,
  title={{Tsallis-INF}: An Optimal Algorithm for Stochastic and Adversarial Bandits},
  author={Zimmert, Julian and Seldin, Yevgeny},
  journal={Journal of Machine Learning Research},
  volume={22},
  number={28},
  pages={1--49},
  year={2021},
  url={https://jmlr.org/papers/v22/19-753.html}
}

@article{blackwell1956analog,
  title={An Analog of the Minimax Theorem for Vector Payoffs},
  author={Blackwell, David},
  journal={Pacific Journal of Mathematics},
  volume={6},
  number={1},
  pages={1--8},
  year={1956},
  doi={10.2140/pjm.1956.6.1}
}

@incollection{hannan1957approximation,
  title={Approximation to {Bayes} Risk in Repeated Play},
  author={Hannan, James},
  booktitle={Contributions to the Theory of Games III},
  series={Annals of Mathematics Studies},
  volume={39},
  pages={97--139},
  publisher={Princeton University Press},
  year={1957}
}

@article{foster1999regret,
  title={Regret in the On-Line Decision Problem},
  author={Foster, Dean P. and Vohra, Rakesh V.},
  journal={Games and Economic Behavior},
  volume={29},
  number={1--2},
  pages={7--35},
  year={1999},
  doi={10.1006/game.1999.0740}
}

@article{hart2000simple,
  title={A Simple Adaptive Procedure Leading to Correlated Equilibrium},
  author={Hart, Sergiu and Mas-Colell, Andreu},
  journal={Econometrica},
  volume={68},
  number={5},
  pages={1127--1150},
  year={2000},
  doi={10.1111/1468-0262.00153}
}

@incollection{hart2001reinforcement,
  title={A Reinforcement Procedure Leading to Correlated Equilibrium},
  author={Hart, Sergiu and Mas-Colell, Andreu},
  booktitle={Economics Essays: A Festschrift for Werner Hildenbrand},
  editor={Debreu, G{\'e}rard and Neuefeind, Wilhelm and Trockel, Walter},
  pages={181--200},
  publisher={Springer},
  year={2001}
}

@article{leslie2005individual,
  title={Individual {Q}-Learning in Normal Form Games},
  author={Leslie, D. S. and Collins, E. J.},
  journal={SIAM Journal on Control and Optimization},
  volume={44},
  number={2},
  pages={495--514},
  year={2005},
  doi={10.1137/S0363012903437976}
}

@article{foster2006regret,
  title={Regret Testing: Learning to Play {Nash} Equilibrium without Knowing You Have an Opponent},
  author={Foster, Dean P. and Young, H. Peyton},
  journal={Theoretical Economics},
  volume={1},
  number={3},
  pages={341--367},
  year={2006},
  url={https://www.econtheory.org/ojs/index.php/te/article/viewArticle/20060341/0}
}

@article{germano2007global,
  title={Global {Nash} Convergence of {Foster} and {Young}'s Regret Testing},
  author={Germano, Fabrizio and Lugosi, G{\'a}bor},
  journal={Games and Economic Behavior},
  volume={60},
  number={1},
  pages={135--154},
  year={2007},
  doi={10.1016/j.geb.2006.06.001}
}
\bibliographystyle{plainnat}

\clearpage
\appendix
\section{Proof of Theorem~\ref{thm:regret}}\label{app:analysis}

We first analyze a run with planned length $H$, conditional on any history
at its start. Lemma~\ref{lem:completion-sampling} controls the cost of
optimistic completion and gives simultaneous sampling bounds.
Lemmas~\ref{lem:local-geometry} and~\ref{lem:coefficient-bounds} then show
that the estimated model recovers the relevant rows and columns and
provides valid coefficient intervals.
The main regret argument consists of two complementary bounds:
Lemma~\ref{lem:visit-counts} controls the accumulated estimation cost,
while Lemma~\ref{lem:epoch-regret} gives a negative payoff contribution
that absorbs this cost. We combine these lemmas in the final proof of
Theorem~\ref{thm:regret}.

Throughout the appendix, $t$ is the round index within a run, and
$\ell,\varepsilon,h,\tau$ are its algorithm parameters.
The symbol $K_A$ denotes a finite positive constant depending only on
$A$, which may increase between displays. Constants are uniform over
the finitely many completed matrices, row sets, reference columns, and
ordered lists of selected columns. Every threshold $H_A$ below is finite
and depends only on $A$; we enlarge it when necessary.

\subsection{Optimistic completion and sampling bounds}\label{app:concentration}

For a mask $M\subseteq[n]\times[m]$, recall the true completed matrix
\[
  A^M_{ij}=
  \begin{cases}A_{ij},&(i,j)\in M,\\1,&(i,j)\notin M.\end{cases}
\]
Let $C_{ij}(t)$ count observations through round $t$ and let
$\widehat A_{ij}(c)$ be the empirical mean of the first $c$ observations
of entry $(i,j)$. We also define the cumulative sampling probability
\[
  Q_{ij}(t)=\sum_{u=1}^t p_{u,i}\mathbf1\{J_u=j\}.
\]
The following lemma separates the cost of entries that have not yet
been acquired from the regret in the completed games.

\begin{lemma}[Completion cost and sampling bounds]\label{lem:completion-sampling}
For every prefix $s\le H$, let
\[
  M_t=\{(i,j):C_{ij}(t-1)\ge h\},\qquad
  \mathcal D_s=\sum_{t=1}^s
       \bigl(v(A^{M_t})-p_t^\top A^{M_t}_{J_t}\bigr).
\]
The following statements hold.
\begin{enumerate}
\item[(i)] The expected Nash regret satisfies
\begin{equation}\label{eq:completion-comparison}
  s\,v(A)-\mathbb E\sum_{t=1}^s r_t
  \le\mathbb E\mathcal D_s+2nmh.
\end{equation}
\item[(ii)] There is an event $\mathcal G$ with
$\mathbb P(\mathcal G^c)\le nm(2H+1)e^{-\ell}\le H^{-2}$
on which
\begin{equation}\label{eq:confidence-event}
  |\widehat A_{ij}(c)-A_{ij}|\le\sqrt{2\ell/c},
  \qquad C_{ij}(t)\ge\tfrac12Q_{ij}(t)-\ell
\end{equation}
simultaneously for every entry, every positive observation count $c$
attained by round $H$, and every $t\le H$.
\item[(iii)] The number of epochs intersecting any prefix is at most
\[
  1+nm\max\{0,1+\lfloor\log_2(H/h)\rfloor\}
  =\mathcal O_{n,m}(1+\log H).
\]
\end{enumerate}
\end{lemma}

\begin{proof}
Let $\mathcal H_t$ contain the information $\mathcal F_t^-$ from the
feedback model and the realized opponent action $J_t$, before drawing
$I_t$. Conditional independence of the actions implies
\[
  \mathbb P(I_t=i\mid\mathcal H_t)=p_{t,i},
  \qquad
  \mathbb E[r_t\mid\mathcal H_t,I_t]=A_{I_tJ_t}.
\]

\paragraph{Completion cost.}
The algorithm rebuilds immediately after an entry reaches $h$
observations, so $M_t$ is the mask used at round $t$.
Since $A^{M_t}\ge A$ entrywise, we have $v(A^{M_t})\ge v(A)$ and
\[
  v(A)-p_t^\top A_{J_t}
  \le v(A^{M_t})-p_t^\top A^{M_t}_{J_t}
       +p_t^\top(A^{M_t}-A)_{J_t}.
\]
Conditioning on $\mathcal H_t$ and then summing gives
\[
\begin{aligned}
  \mathbb E\sum_{t=1}^s p_t^\top(A^{M_t}-A)_{J_t}
  &=\mathbb E\sum_{t=1}^s
       (1-A_{I_tJ_t})\mathbf1\{C_{I_tJ_t}(t-1)<h\}\\
  &\le2nmh.
\end{aligned}
\]
Each entry contributes at most $h$ observations, each with cost at most
two. The tower property also gives
$\mathbb E r_t=\mathbb E[p_t^\top A_{J_t}]$, proving (i).
No conditioning on $\mathcal G$ is used in this argument.

\paragraph{Payoff concentration.}
For a fixed entry, set
\[
  X_t=\mathbf1\{I_t=i,J_t=j\}(r_t-A_{ij}),\qquad
  Y_t=\mathbf1\{I_t=i,J_t=j\}.
\]
Conditional Hoeffding's lemma gives
\[
  \mathbb E\!\left[
    \exp\{\lambda X_t-\lambda^2Y_t/2\}
    \,\middle|\,\mathcal H_t,I_t\right]\le1.
\]
Thus $\exp(\lambda\sum_{u\le t}X_u-\lambda^2C_{ij}(t)/2)$
is a nonnegative supermartingale starting at one.
We stop it at the $c$th observation of this entry or at $H$,
whichever comes first. On the event that the $c$th observation occurs,
choosing $\lambda=\sqrt{2\ell/c}$ bounds the probability of
$\widehat A_{ij}(c)-A_{ij}>\sqrt{2\ell/c}$ by $e^{-\ell}$.
The negative choice of $\lambda$ gives the lower tail.
A union bound over entries and $c\in[H]$ costs at most $2nmH e^{-\ell}$.

\paragraph{Count concentration.}
For the same entry,
$\exp(Q_{ij}(t)/2-C_{ij}(t))$ is a nonnegative supermartingale.
Indeed, when $J_t=j$, its conditional expected multiplier is
\[
  e^{p_{t,i}/2}(1-p_{t,i}+p_{t,i}e^{-1})
  \le e^{p_{t,i}(1/2-(1-e^{-1}))}\le1;
\]
when $J_t\ne j$, the multiplier is one.
Ville's inequality and a union bound over entries imply
$Q_{ij}(t)/2-C_{ij}(t)\le\ell$ for all $i,j,t$ except on an event of
probability at most $nm e^{-\ell}$.
Together with the payoff bounds, this proves (ii), including at random
epoch boundaries.

\paragraph{Number of epochs.}
After the initial epoch, every reconstruction is triggered by an entry
count reaching $h,2h,4h,\ldots$. Each entry crosses each such threshold
at most once, and no count exceeds $H$. Counting these thresholds proves
(iii).
\end{proof}

\subsection{Identifying the optimal support and binding columns}
\label{app:geometry}

For a matrix $G$, let
$X(G)=\{x\in\Delta_n:G^\top x\ge v(G)\mathbf1\}$ be its optimal row set.
For a completion $A^M$, write $v_M=v(A^M)$, $X^M=X(A^M)$, and
\[
  I_*^M=\{i:\max_{x\in X^M}x_i>0\},\qquad
  J_*^M=\{j:x^\top A_j^M=v_M\text{ for every }x\in X^M\}.
\]
The next lemma shows that the empirical construction identifies these
sets and leaves room for the local strategy updates.

\begin{lemma}[Support identification and feasibility of local strategies]
\label{lem:local-geometry}
There exist $H_A<\infty$ and $\beta_A>0$ such that, on $\mathcal G$,
every epoch of a run with $H\ge H_A$ has the following properties.
\begin{enumerate}
\item[(i)] The selected sets satisfy $I=I_*^M$ and $J=J_*^M$.
For $B=A^M[I,:]$, we have $v(B)=v_M$ and an optimal row strategy with
all coordinates positive. The zero extension of $\bar p$ is within
$K_A\varepsilon$ of $X^M$ in $\ell_1$.
\item[(ii)] The selected columns form the greedy independent list for
the true projected differences, in the algorithm's count order.
For $b>0$, $\max_l\|\widehat R_{\cdot l}\|_1\le K_A$ and
\begin{equation}\label{eq:small-certificate}
  \chi:=2\sum_{l=1}^b e_l\|\widehat R_{\cdot l}\|_1
  \le K_A\varepsilon.
\end{equation}
\item[(iii)] For every $z\in[-1,1]^b$, the vector
$x(z)=\widehat x+4\widehat R E z$ satisfies
\[
  \mathbf1^\top x(z)=1,\qquad
  x_i(z)\ge\beta_A\ (i\in I),\qquad
  x(z)^\top B_j-v_M\ge\beta_A\ (j\notin J).
\]
For $b=0$, the same conclusions hold with $x(z)$ replaced by $\bar p$.
In particular, the simplex projection in \eqref{eq:local-update} is inactive.
\end{enumerate}
\end{lemma}

\begin{proof}
We establish the distance estimate used to identify $I,J$, then analyze
the selected differences and the resulting strategies.

\paragraph{Distance to the optimal set.}
For each fixed $G\in[-1,1]^{n\times m}$, there is a finite $L_G$ such that
\begin{equation}\label{eq:polyhedral-distance}
  \inf_{y\in X(G)}\|x-y\|_1
  \le L_G\bigl(v(G)-\min_jx^\top G_j\bigr)_+
  \qquad(x\in\Delta_n).
\end{equation}
Here is a direct finite-polytope proof of this form of Hoffman's bound
\citep{hoffman1952approximate}.
Consider
\[
  \mathcal P_G=\{(x,a):x\in\Delta_n,\ 0\le a\le2,\
                         G^\top x+a\mathbf1\ge v(G)\mathbf1\}.
\]
Let $\eta_G>0$ be the smallest positive last coordinate among its
vertices. Such a coordinate exists because the face $a=2$ is nonempty.
In a vertex decomposition of $(x,a)\in\mathcal P_G$, the total weight of
vertices with positive last coordinate is at most $a/\eta_G$.
Replacing their row coordinates by any fixed point of $X(G)$ gives
a point of $X(G)$ within $2a/\eta_G$ of $x$ in $\ell_1$.
Taking $a=(v(G)-\min_jx^\top G_j)_+$ proves
\eqref{eq:polyhedral-distance}.

On $\mathcal G$, each acquired entry has error at most
$\sqrt{2\ell/h}\le\varepsilon$, and each artificial entry agrees
exactly with $A^M$. The value is Lipschitz in the largest entrywise error,
so
\[
  |\widehat v-v_M|\le\varepsilon,\qquad
  X^M\subseteq Z\subseteq
  \{x\in\Delta_n:(A^M)^\top x\ge(v_M-4\varepsilon)\mathbf1\}.
\]
Equation~\eqref{eq:polyhedral-distance} now places every point in $Z$
within $K_A\varepsilon$ of $X^M$.

\paragraph{Identification by the joint center.}
We use a property of the logarithmic objective. For nonnegative affine
functions $f_1,\ldots,f_a$ on a compact polytope $P$, a maximizer $x^c$
of $\sum_r\log(\varepsilon+f_r(x))$ satisfies
\begin{equation}\label{eq:center-feature}
  \sum_{r=1}^a\frac{\varepsilon+f_r(w)}
                      {\varepsilon+f_r(x^c)}\le a
  \quad(w\in P),\qquad
  \varepsilon+f_r(x^c)\ge
  \frac{\varepsilon+\max_{w\in P}f_r(w)}a.
\end{equation}
The first inequality follows from the nonpositive directional derivative
at $x^c$ toward $w$. Keeping one positive summand and maximizing its
numerator gives the second inequality.

We apply \eqref{eq:center-feature} to the $n+m$ features in
\eqref{eq:center}. For every $x\in\Delta_n$,
\[
  0\le s_j(x)-(x^\top A_j^M-v_M)\le4\varepsilon.
\]
If a row probability or true column slack vanishes throughout $X^M$,
its empirical feature is at most $K_A\varepsilon$ on $Z$.
Every other feature has a strictly positive maximum on $X^M$;
since $X^M\subseteq Z$, \eqref{eq:center-feature} bounds its value at
$p^c$ below by a positive constant for sufficiently small $\varepsilon$.
Because $\varepsilon/\tau\to0$ and $\tau\to0$ as $H\to\infty$,
the tests in \eqref{eq:face} identify $I_*^M$ and $J_*^M$.

Averaging finitely many optimal strategies that witness the positive
features yields an optimum with positive probabilities on all rows
in $I_*^M$ and positive slack at every column outside $J_*^M$.
In particular $J_*^M\ne\varnothing$, since otherwise this optimum
would beat $v_M$ against every column.
The center assigns only $\mathcal O_A(\varepsilon)$ total mass to
rows outside $I_*^M$. Removing mass $a$ and renormalizing changes a
distribution by $2a$ in $\ell_1$, so $\bar p$ remains within
$K_A\varepsilon$ of $X^M$ and keeps positive row probabilities and
nonbinding slacks bounded away from zero.
Restricting to $I_*^M$ preserves the value and admits an optimum with
all coordinates positive. This proves (i).

\paragraph{Recovery of independent differences.}
For the test in \eqref{eq:certificate}, let $D'$ be the true counterpart
of the trial matrix $\widehat D'$.
If a trial includes a column with $c_j=\infty$, its entries on $I$
are all artificial. The reference also has infinite count, so this
candidate's difference is zero and the trial is rejected.
We may therefore restrict attention to positive trial error scales
$e'_l$, with diagonal matrix $E'$.

When $G'=P_0\widehat D'$ has full rank, let
$F'=(D')^\top R'-\mathrm{Id}$.
The reference has at least as many effective observations as each
candidate, so
\[
  |F'_{lk}|\le2e'_l\|R'_{\cdot k}\|_1,\qquad
  \|(E')^{-1}F'E'\|_\infty\le\chi'.
\]
The matrix infinity norm here is the maximum absolute row sum.
If the trial is accepted, $\chi'\le1/2$, hence
$\mathrm{Id}+F'$ is invertible.
Since the columns of $R'$ have zero sum,
$(D')^\top R'=(P_0D')^\top R'$; invertibility implies that the
true projected differences are independent.

Conversely, there are only finitely many true trial matrices and column
orders. Their independent projected matrices have a positive minimum
nonzero singular value. For sufficiently small $\varepsilon$, each
such empirical trial remains independent, its map $R'$ is bounded,
and its certificate is $\mathcal O_A(\varepsilon)\le1/2$.
Induction through the ordered candidates thus recovers the true greedy
list. The same boundedness gives \eqref{eq:small-certificate}, proving (ii).

\paragraph{Feasibility of local strategies.}
Every selected true difference has zero payoff at every optimum.
Part (i) and the entry confidence bounds consequently give
$\|\widehat D^\top\bar p\|_\infty\le K_A\varepsilon$.
Part (ii) implies
\[
  \|\widehat x-\bar p\|_1\le K_A\varepsilon,\qquad
  \sup_{z\in[-1,1]^b}\|4\widehat R E z\|_1\le K_A\varepsilon.
\]
Also $\mathbf1^\top\widehat R=0$, so every $x(z)$ has total mass one.
The positive probabilities and nonbinding slacks of $\bar p$ therefore
remain bounded below by a common $\beta_A>0$ throughout this set.
The same properties already hold for $\bar p$ when $b=0$.
This proves (iii). Finiteness of the family of completions allows
one choice of $H_A,\beta_A$ for all epochs.
\end{proof}

\subsection{Confidence bounds for the true coefficients}\label{app:coefficients}

We next identify the coefficients estimated by
$\widehat\alpha_j=\widehat R^\top\widehat B_j$ and show why shrinking
them to the nearest point to zero suppresses updates caused only by
estimation error.

\begin{lemma}[Coefficient accuracy and zero preservation]
\label{lem:coefficient-bounds}
On $\mathcal G$, for $H\ge H_A$ and every epoch with $b>0$, write
$B=A^M[I,:]$ and $D=[D_1,\ldots,D_b]$, where
$D_l=B_{j_l}-B_{j_0}$.
Each binding column has unique coefficients $\alpha_j$ satisfying
\begin{equation}\label{eq:true-column-model}
  B_j=v(B)\mathbf1+D\alpha_j\qquad(j\in J).
\end{equation}
These coefficients are uniformly bounded by $K_A$.
With $\delta_j=\sqrt{2\ell/c_j}$, taking $\delta_j=0$ if $c_j=\infty$,
the confidence intervals satisfy
$|\widehat\alpha_{lj}-\alpha_{lj}|\le b_{lj}$.
The corrected coefficients have the true sign whenever nonzero and obey
\begin{equation}\label{eq:zero-preservation}
  \alpha_{lj}=0\ \Longrightarrow\ \widetilde\alpha_{lj}=0,
  \qquad |\widetilde\alpha_{lj}|\le|\alpha_{lj}|,
\end{equation}
as well as
\begin{equation}\label{eq:gradient-error}
  |\widetilde\alpha_{lj}-\alpha_{lj}|
  \le K_A\left(\delta_j+\sum_k e_k|\alpha_{kj}|\right).
\end{equation}
If $b=0$, every binding column is $v(B)\mathbf1$.
\end{lemma}

\begin{proof}
We first establish the exact representation, using
Lemma~\ref{lem:local-geometry}.
For every $j\in J$, the projected difference
$P_0(B_j-B_{j_0})$ is in the span of $P_0D$.
Thus $B_j-B_{j_0}=D\gamma_j+a_j\mathbf1$ for some $\gamma_j,a_j$.
Evaluating at a true optimal row strategy gives $a_j=0$, because
all binding columns have value $v(B)$ and every column of $D$ has
zero payoff at that strategy.

For a dual optimum $y$, complementary slackness implies
$\operatorname{supp}(y)\subseteq J$.
Moreover, $By\le v(B)\mathbf1$, and an optimal row strategy with all
coordinates positive forces equality in every row.
Consequently
\[
  v(B)\mathbf1=By
  =B_{j_0}+D\sum_{j\in J}y_j\gamma_j,
  \qquad
  B_j=v(B)\mathbf1+
       D\left(\gamma_j-\sum_{k\in J}y_k\gamma_k\right).
\]
This proves existence in \eqref{eq:true-column-model}.
Independence of $P_0D$ gives uniqueness.
The true matrices and selected lists form a finite family, so their
coefficients have a common finite bound.
For $b=0$, the same argument gives $B_j=v(B)\mathbf1$.

For $b>0$, set $F=D^\top\widehat R-\mathrm{Id}_b$ and
$\nu_l=\|\widehat R_{\cdot l}\|_1$. The confidence bounds give
\begin{equation}\label{eq:coordinate-error}
  \|D_l-\widehat D_l\|_\infty\le2e_l,\qquad
  |F_{lk}|\le2e_l\nu_k\le K_Ae_l.
\end{equation}
Using \eqref{eq:true-column-model} and
$\mathbf1^\top\widehat R=0$, we obtain the exact identity
\[
  \widehat\alpha_j-\alpha_j
  =F^\top\alpha_j+\widehat R^\top(\widehat B_j-B_j).
\]
Let $Q_j=\sum_k e_k|\alpha_{kj}|$ and
$\widehat Q_j=\sum_k e_k|\widehat\alpha_{kj}|$.
It follows that
\[
  |\widehat\alpha_{lj}-\alpha_{lj}|
       \le\nu_l(\delta_j+2Q_j),\qquad
  Q_j\le\widehat Q_j+\frac\chi2\delta_j+\chi Q_j.
\]
Since $\chi\le1/2$, these inequalities imply
\[
  |\widehat\alpha_{lj}-\alpha_{lj}|
  \le\frac{\nu_l(\delta_j+2\widehat Q_j)}{1-\chi}
  \le2\nu_l(\delta_j+2\widehat Q_j)=b_{lj}.
\]
Thus the interval used by the algorithm contains $\alpha_{lj}$.
The shrinkage in \eqref{eq:conservative} chooses its nearest point to
zero, which lies between zero and $\alpha_{lj}$,
which proves \eqref{eq:zero-preservation} and the sign assertion.
Its distance from $\alpha_{lj}$ is at most $2b_{lj}$.
The reverse triangle inequality also gives
\[
  \widehat Q_j\le(1+\chi)Q_j+\frac\chi2\delta_j.
\]
Substituting this into $2b_{lj}$ and using $\nu_l\le K_A$ proves
\eqref{eq:gradient-error}.
If $c_j=\infty$, then $B_j=\widehat B_j=\mathbf1$;
as this column is binding, $v(B)=1$ and both coefficient vectors are zero.
\end{proof}

\subsection{Column visits and accumulated estimation error}
\label{app:visits}

We now fix an epoch and a prefix $\mathcal T$ of its rounds.
Let $N_j=\sum_{t\in\mathcal T}\mathbf1\{J_t=j\}$.
For $b>0$, define the active columns and their visit counts by
\[
  S_l=\{j\in J:\alpha_{lj}\ne0\},\qquad
  N_l=\sum_{j\in S_l}N_j,\qquad
  q_l=c_{j_l},\qquad d_l=\sum_{j\in J}N_j\alpha_{lj}.
\]
The two quantities used in the regret argument are
\[
  V=\sum_l e_l^2N_l,\qquad U=\sum_l e_l|d_l|.
\]
Here $V$ measures the accumulated estimation cost, while $U$ measures
the payoff advantage available from a persistent imbalance in the
opponent's columns. We set $U=V=0$ for $b=0$.
For visit bounds, extend $\delta_j=\sqrt{2\ell/c_j}$ to all columns,
again taking zero for infinite counts.

\begin{lemma}[Column visit bounds and accumulated estimation error]\label{lem:visit-counts}
On $\mathcal G$, for $H\ge H_A$, every epoch prefix satisfies
\begin{equation}\label{eq:epoch-visits}
  N_j\le\frac{6c_j}{\beta_A}\quad(c_j<\infty),\qquad
  \sum_j\delta_j^2N_j\le\frac{12m\ell}{\beta_A},
\end{equation}
where $\beta_A$ is from Lemma~\ref{lem:local-geometry}.
For $b>0$, the active visit counts obey
$N_l\le K_A(q_l+|d_l|)$.
Consequently, there is a finite $a_A$ depending only on $A$ such that
\begin{equation}\label{eq:active-count}
  V\le a_A\ell+a_A\varepsilon U.
\end{equation}
\end{lemma}

The bound \eqref{eq:active-count} is stronger than a direct bound by
the epoch length: many active visits must either exhaust the available
samples or create a large $U$, which the next lemma uses to reduce regret.

\begin{proof}
We use the count thresholds to bound column visits and the order of the
selected differences to control the active counts.

\paragraph{Visits before reconstruction.}
For a column with finite $c_j$, choose an acquired row $i\in I$
attaining its minimum count at the start of the epoch.
The next threshold for this entry is at most $2c_j$.
Throughout the epoch, including its final round, its count is therefore
at most $2c_j$. If $u$ is the last round of the prefix,
Lemmas~\ref{lem:completion-sampling} and~\ref{lem:local-geometry} give
\[
  \frac{\beta_A N_j}{2}
  \le\frac{Q_{ij}(u)}2
  \le C_{ij}(u)+\ell
  \le2c_j+\ell.
\]
Although $Q_{ij}(u)$ includes observations before this epoch, these
only increase it. Since $c_j\ge h\ge\ell$, we obtain
$N_j\le6c_j/\beta_A$.
Multiplying by $\delta_j^2=2\ell/c_j$ and summing over finite-count
columns proves \eqref{eq:epoch-visits}; infinite-count columns contribute
zero.

\paragraph{Active visits and signed coefficients.}
Fix a coordinate $l$. If $j\in J$ has $c_j>q_l$, it was processed
before $j_l$, or is the reference column.
By Lemma~\ref{lem:local-geometry}, its projected difference is spanned
by the previously accepted differences.
The constant-vector remainder again vanishes at a true optimum.
Uniqueness in \eqref{eq:true-column-model} therefore gives
\[
  \alpha_{lj}=\alpha_{l j_0}\qquad(c_j>q_l,\ j\in J).
\]
Let $N_{\rm low}=\sum_{j\in J:c_j\le q_l}N_j$.
Equation~\eqref{eq:epoch-visits} yields $N_{\rm low}\le K_Aq_l$.
If $\alpha_{l j_0}=0$, only these columns can be active, so
$N_l\le K_Aq_l$.

Otherwise, all columns with $c_j>q_l$ share the same nonzero coefficient.
Writing $N_{\rm high}=\sum_{j\in J:c_j>q_l}N_j$, we have
\[
  d_l=\alpha_{l j_0}N_{\rm high}
        +\sum_{j\in J:c_j\le q_l}N_j\alpha_{lj}.
\]
The coefficients are bounded, and their nonzero values have a positive
minimum over the finite family in Lemma~\ref{lem:coefficient-bounds}.
It follows that
\[
  N_{\rm high}\le K_A(|d_l|+N_{\rm low}),\qquad
  N_l\le N_{\rm high}+N_{\rm low}\le K_A(|d_l|+q_l).
\]
Finally, $e_l^2q_l=2\ell$ and $e_l\le\varepsilon$ imply
\[
  V\le K_A\sum_l e_l^2q_l+K_A\sum_l e_l^2|d_l|
    \le a_A\ell+a_A\varepsilon U
\]
for a suitable $a_A$, proving \eqref{eq:active-count}.
For $b=0$, the last inequality holds with $U=V=0$.
\end{proof}

\subsection{Regret within an epoch}\label{app:local-regret}

The next lemma isolates the contribution that offsets $V$.
For an epoch prefix $\mathcal T$, define its completed-game deficit by
\[
  \mathcal R_{\rm epoch}
  =\sum_{t\in\mathcal T}\bigl(v(B)-p_t^\top B_{J_t}\bigr),
  \qquad B=A^M[I,:].
\]
In expressions involving $B$, we identify $p_t$ with its restriction
to $I$.

\begin{lemma}[Epoch regret with a payoff advantage]\label{lem:epoch-regret}
There are finite constants $b_A,H_A$ depending only on $A$ such that,
on $\mathcal G$, every epoch prefix of a run with $H\ge H_A$ satisfies
\begin{equation}\label{eq:epoch-regret}
  \mathcal R_{\rm epoch}\le b_A\ell+b_AV-U,
\end{equation}
where $U,V$ are the quantities in Lemma~\ref{lem:visit-counts}.
If $b=0$, then $\mathcal R_{\rm epoch}\le0$.
\end{lemma}

\begin{proof}
When $b=0$, Lemma~\ref{lem:coefficient-bounds} makes every binding
column equal to $v(B)\mathbf1$, while Lemma~\ref{lem:local-geometry}
gives positive surplus against every other column.
Each round's deficit is therefore nonpositive.
We henceforth suppose $b>0$.

\paragraph{An optimal comparison strategy.}
Let $F=D^\top\widehat R-\mathrm{Id}_b$ as in the proof of
Lemma~\ref{lem:coefficient-bounds}, and define
\[
  p^*=\widehat x-\widehat R(\mathrm{Id}_b+F)^{-1}D^\top\widehat x.
\]
By \eqref{eq:coordinate-error}, $F=\mathcal O_A(\varepsilon)$,
so the inverse is bounded for sufficiently large $H$.
Since $\widehat D^\top\widehat x=0$,
$D^\top\widehat x=\mathcal O_A(\varepsilon)$.
The correction from $\widehat x$ to $p^*$ thus has size
$\mathcal O_A(\varepsilon)$ and total mass zero.
Lemma~\ref{lem:local-geometry} implies that $p^*$ remains a probability
distribution with positive surplus at every nonbinding column.
Also $D^\top p^*=0$, so \eqref{eq:true-column-model} gives
$(p^*)^\top B_j=v(B)$ for every $j\in J$.
Consequently $p^*$ is a true optimum.

Because $p^*-\widehat x$ is in the range of $\widehat R$, we can write
\[
  p^*=\widehat x+4\widehat R E z^*,\qquad
  z_l^*=\frac{\widehat D_l^\top p^*}{4e_l}
       =\frac{(\widehat D_l-D_l)^\top p^*}{4e_l}.
\]
Equation~\eqref{eq:coordinate-error} gives $|z_l^*|\le1/2$.
For the actual iterate $z_t$, define
\[
  u_l(t)=z_l^*+\sum_kF_{lk}\frac{e_k}{e_l}(z_k^*-z_{k,t}).
\]
This adjustment accounts for using empirical rather than true
payoff differences. In fact,
\begin{equation}\label{eq:payoff-comparison}
  (p^*-p_t)^\top B_{J_t}
  =4\sum_l e_l\alpha_{lJ_t}(u_l(t)-z_{l,t})
  \qquad(J_t\in J).
\end{equation}
To verify the identity, substitute
$p^*-p_t=4\widehat R E(z^*-z_t)$ and
$D^\top\widehat R=\mathrm{Id}_b+F$ into
\eqref{eq:true-column-model}.

\paragraph{Feasibility and variation of the comparator.}
The matrix $K=E^{-1}FE$ satisfies
$|K_{lk}|\le2e_k\nu_k$, so $\|K\|_\infty\le\chi$.
By \eqref{eq:small-certificate}, we can choose $H_A$ large enough that
$\chi\le1/6$. Hence
\[
  \|u(t)\|_\infty
  \le\|z^*\|_\infty+\chi\|z^*-z_t\|_\infty
  \le\frac12+\frac32\chi\le\frac34.
\]
Let $\operatorname{TV}$ denote the sum of absolute changes between
consecutive rounds of the prefix. Lemma~\ref{lem:coefficient-bounds}
bounds all stored coefficients and, by \eqref{eq:zero-preservation},
coordinate $k$ changes only when $J_t\in S_k$.
Nonexpansiveness of clipping gives
$\operatorname{TV}(z_k)\le K_Ae_kN_k$, and therefore
\begin{equation}\label{eq:comparator-variation}
  \operatorname{TV}(u_l)
  \le K_A\sum_k e_k\operatorname{TV}(z_k)
  \le K_A\sum_k e_k^2N_k=K_AV.
\end{equation}

\paragraph{Projected gradient comparison.}
For scalar updates $z_{a+1}=\Pi_{[-1,1]}(z_a+g_a)$ and
comparators $c_a\in[-1,1]$, nonexpansiveness gives
\[
  (c_a-z_a)g_a
  \le\frac12\bigl((z_a-c_a)^2-(z_{a+1}-c_a)^2\bigr)
       +\frac12g_a^2.
\]
Changing the comparator by an amount $d$ changes a squared distance
by at most $4|d|$. Summing the inequalities thus yields
\begin{equation}\label{eq:moving-gradient}
  \sum_a(c_a-z_a)g_a
  \le2+2\operatorname{TV}(c)+\frac12\sum_a g_a^2.
\end{equation}
This is the scalar moving-comparator argument of
\citet{zinkevich2003online}.

For each $l$, apply \eqref{eq:moving-gradient} on the subsequence of
visits to $S_l$, using
\[
  g_t=4e_l\widetilde\alpha_{lJ_t},\qquad
  c_t=u_l(t)+\frac14\operatorname{sgn}(d_l).
\]
The comparator belongs to $[-1,1]$ because $|u_l(t)|\le3/4$.
Between these visits the coordinate does not change.
Subsequence variation is at most the full variation in
\eqref{eq:comparator-variation}, and
$\sum_t g_t^2\le K_Ae_l^2N_l$.
Summing over coordinates, with empty subsequences contributing zero,
bounds the total estimated comparison by $K_A(1+V)$.

\paragraph{Replacing estimated coefficients by true coefficients.}
Since $|c_t-z_{l,t}|\le2$, \eqref{eq:gradient-error} bounds the
replacement error by a constant times
\[
  \sum_{l,t\in\mathcal T:J_t\in S_l}
       e_l\left(\delta_{J_t}
                   +\sum_k e_k|\alpha_{kJ_t}|\right).
\]
For the first term, $2e_l\delta_j\le e_l^2+\delta_j^2$ gives
\[
  \sum_{l,t\in\mathcal T:J_t\in S_l}e_l\delta_{J_t}
  \le K_A\left(V+\sum_j\delta_j^2N_j\right)
  \le K_A(V+\ell),
\]
where the last inequality is \eqref{eq:epoch-visits}.
For the second term, let $N_{lk}$ count visits to $S_l\cap S_k$.
The coefficients are bounded and vanish outside $S_k$, and
\[
  e_le_kN_{lk}
  \le\frac12(e_l^2N_l+e_k^2N_k).
\]
Summing over $l,k$ therefore bounds this term by $K_AV$.
Combining these estimates with the projected gradient bound gives
\[
  \sum_{t\in\mathcal T:J_t\in J}
       4\sum_l e_l\alpha_{lJ_t}(u_l(t)-z_{l,t})
       +\sum_l e_l|d_l|
  \le K_A\ell+K_AV.
\]
Here the second term arises from the comparator shift:
$4e_l\cdot\frac14\operatorname{sgn}(d_l)
\sum_{t\in\mathcal T:J_t\in S_l}\alpha_{lJ_t}=e_l|d_l|$.
By \eqref{eq:payoff-comparison}, the first term is exactly the
completed-game deficit on binding rounds.
Nonbinding rounds have nonpositive deficit by
Lemma~\ref{lem:local-geometry}. This proves \eqref{eq:epoch-regret}
with a suitable $b_A$.
\end{proof}

\subsection{Proof of the main theorem}\label{app:global}

\begin{proof}[Proof of Theorem~\ref{thm:regret}]
We combine the two epoch inequalities, sum over epochs, and then apply
the restart schedule.

\paragraph{An epoch contributes at most $\mathcal O_A(\ell)$.}
Work on $\mathcal G$ in a run with $H\ge H_A$.
Lemmas~\ref{lem:visit-counts} and~\ref{lem:epoch-regret} give
\[
\begin{aligned}
  \mathcal R_{\rm epoch}
  &\le b_A\ell+b_AV-U\\
  &\le b_A(1+a_A)\ell+(a_Ab_A\varepsilon-1)U.
\end{aligned}
\]
Since $\varepsilon=\ell^{-1/2}\to0$, we can enlarge $H_A$ so that
$a_Ab_A\varepsilon\le1$. As $U\ge0$, every epoch prefix then has
$\mathcal R_{\rm epoch}\le K_A\ell$.
This also includes $b=0$, where the deficit is nonpositive.

\paragraph{A run prefix contributes at most $\mathcal O_A(\log^2(eH))$.}
Lemma~\ref{lem:local-geometry} gives $v(B)=v(A^M)$ in every epoch,
so the epoch deficits sum to $\mathcal D_s$ from
Lemma~\ref{lem:completion-sampling}.
The epoch count in that lemma yields
$\mathcal D_s\le K_A\ell(1+\log H)$ on $\mathcal G$.
On every observation sequence, $\mathcal D_s\le2s$ because completed
payoffs lie in $[-1,1]$ and the played strategies are distributions.
By \eqref{eq:completion-comparison} and the failure probability in
Lemma~\ref{lem:completion-sampling}, we obtain
\begin{equation}\label{eq:run-prefix}
\begin{aligned}
  s\,v(A)-\mathbb E\sum_{t=1}^s r_t
  &\le K_A\ell(1+\log H)+2s\,\mathbb P(\mathcal G^c)+2nmh\\
  &\le C(A)\log^2(eH).
\end{aligned}
\end{equation}
We used $\ell=\mathcal O_{n,m}(\log(eH))$,
$h=\mathcal O(\ell^2)$, and $s\le H$.
For $H<H_A$, the bound
$s\,v(A)-\mathbb E\sum_{t=1}^s r_t\le2H$ is covered by increasing
the same finite constant $C(A)$.
Thus \eqref{eq:run-prefix} holds for every planned length and every
prefix, uniformly over permitted opponents and reward processes.

\paragraph{Combining the runs.}
All preceding bounds hold conditionally on any history at the start
of a fresh run. If global time $T$ falls in run $r$, we apply
\eqref{eq:run-prefix} to each completed run and the current prefix.
Taking expectations gives
\[
  R_T\le C(A)\sum_{k=0}^r\log^2(eH_k)
       \le K_A\log^2(eH_r),\qquad H_k=2^{2^k}.
\]
The last inequality follows from
$\log(eH_k)=1+2^k\log2$ and the geometric sum $\sum_{k=0}^r4^k$.
For $r\ge1$, at least $H_{r-1}$ rounds precede run $r$, so
\[
  \log H_r=2\log H_{r-1}\le2\log T.
\]
The first run has length $H_0=2$ and contributes only a constant.
Renaming the constant proves
$R_T\le C(A)\log^2(eT)$ for every integer $T\ge1$.
\end{proof}

\end{document}